%% file: RobustRNN_LCSS_Arxiv.tex
\documentclass[letterpaper, 10 pt, conference]{ieeeconf}  % Comment this line out
\IEEEoverridecommandlockouts                              % This command is only
\usepackage{amssymb}
\usepackage{graphicx}
\graphicspath{{./figure/}}
\usepackage{float}
\usepackage{color}
\usepackage{mathrsfs} 
\usepackage{mathtools}% Loads amsmath\usepackage{amssymb}
\usepackage{caption}
\usepackage{subcaption}

\usepackage{hyperref}
\usepackage{cleveref}

\usepackage{amsthm}

\theoremstyle{plain}
\newtheorem{theorem}{Theorem}

\newtheorem{remark}{Remark}

\theoremstyle{definition}
\newtheorem{definition}{Definition}

\definecolor{tbcolor}{rgb}{0.1, 0.2, 0.7}

\definecolor{cadmiumgreen}{rgb}{0.0, 0.42, 0.24}

\newcommand{\R}{\mathbb{R}}
\newcommand{\N}{\mathbb{N}}

\newcommand{\mat}[1]{\begin{matrix}#1\end{matrix}} % no delimiters
\newcommand{\bmat}[1]{\left[\mat{#1}\right]} % square brackets as delimiters

\title{\LARGE \bf A Convex Parameterization of Robust Recurrent Neural Networks}

\author{Max Revay, Ruigang Wang, Ian R. Manchester% <-this % stops a space
	\thanks{This work was supported by the Australian Research Council.}% <-this % stops a space
	\thanks{The authors are with the Australian Centre for Field Robotics, The University of Sydney, Sydney, NSW 2006, Australia
		(e-mail: {\tt\small  ian.manchester@sydney.edu.au}).}%
}

\begin{document}

\maketitle
\thispagestyle{empty}
\pagestyle{empty}

%%%%%%%%%%%%%%%%%%%%%%%%%%%%%%%%%%%%%%%%%%%%%%%%%%%%%%%%%%%%%%%%%%%%%%%%%%%%%%%%
\begin{abstract}
Recurrent neural networks (RNNs) are a class of nonlinear dynamical systems often used to model sequence-to-sequence maps. RNNs have excellent expressive power but lack the stability or robustness guarantees that are necessary for many applications. In this paper, we formulate convex sets of RNNs with stability and robustness guarantees. The guarantees are derived using incremental quadratic constraints and can ensure global exponential stability of all solutions, and bounds on incremental $ \ell_2 $ gain (the Lipschitz constant of the learned sequence-to-sequence mapping). Using an implicit model structure, we construct a parametrization of RNNs that is jointly convex in the model parameters and stability certificate. We prove that this model structure includes all previously-proposed convex sets of stable RNNs as special cases, and also includes all stable linear dynamical systems. We illustrate the utility of the proposed model class in the context of non-linear system identification.
\end{abstract}

\input{Introduction}

\input{Problem}

\input{IQC_RNN}

\input{Examples}

\bibliographystyle{IEEEtran}
\bibliography{ref}

\end{document}

%% file: Introduction.tex
\section{INTRODUCTION}

%\remove{Neural networks are a class of universal function approximator frequently used in machine learning. While there have been some great successes in solving complex tasks, their lack of robustness often prevents their application, particularly in a safety-critical context. This realization has motivated a large amount of research into the area of adversarial defenses \cite{Salman:2020,Goodfellow:2014,Cheng:2018} . This area seeks guarantees of robustness against adversarially chosen inputs.} \tb{One Sentence}

%\remove{Applications are not really clear. There are more than just safety critical applications.}

RNNs are state-space models incorporating neural networks that are frequently used in system identification and machine learning to model dynamical systems and other sequences-to-sequence mappings. It has long been observed that RNNs can be difficult to train in part due to model instability, referred to as the exploding gradients problem \cite{bengio1994learning}, and recent work shows that these models are often not robust to input perturbations \cite{cheng2020seq2sick}. These issues are related to long-standing concerns in control theory, i.e. stability and Lipschitz continuity solutions of dynamical systems \cite{Zames:1966}.

There are many types of stability for nonlinear systems (e.g., RNNs). When learning dynamical systems with inputs, Lyapunov approaches are inappropriate as they require the construction of a Lyapunov function about a known stable solution. In machine learning and system identification, however, the goal is to simulate the learned model with new inputs, generating new solutions.
Incremental stability \cite{Desoer:1975} and contraction analysis \cite{Lohmiller:1998} avoid this issue by showing stability for all inputs and trajectories.

Even if a model is stable, it is usually problematic if its output is very sensitive to small changes in the input.
This sensitivity can be quantified by the model's incremental $\ell_2$ gain. Finite incremental $\ell_2$ gain implies both boundedness and continuity of the input-output map \cite{Desoer:1975}.
Furthermore, the incremental $\ell_2$ gain bound is also a bound on the Lipschitz constant of the sequence-to-sequence mapping. In machine learning, the Lipschitz constant is used in proofs of generalization bounds \cite{Bartlett:2017}, analysis of expressiveness \cite{Zhou:2019} and guarantees of robustness to adversarial attacks \cite{Huster:2018,Qian:2019}. 

%\remove{Need to emphasiz e the importance of convexity?}

%Even for linear systems, identifying systems with guarantees of stability or $\ell_2$ gain bounds is difficult

The problem of training models with stability or robustness guaranteed a-priori has seen significant attention for both linear \cite{lacy2003subspace, miller2013subspace} and nonlinear \cite{Tobenkin:2017,umlauft2017learning,Kolter:2019} models. The main difficult comes from the non-convexity of most model structures and their stability certificates. Some methods deal with this difficulty by simply fixing the stability certificate and optimizing over the model parameters at significant cost to model expressibility \cite{Miller:2019}.
It has recently been shown that implicit parametrizations allow joint convexity of the model and a stability certificate for linear \cite{Umenberger:2018}, polynomial \cite{Tobenkin:2017} and RNN \cite{Revay:2019} models. It has been observed that stability constraints serve as an effective regulariser and can improve generalisation performance \cite{Umenberger:2018specialized, Revay:2019}.

When a system is expressed in terms of a neural network, even the problem of analyzing the stability of known dynamics is difficult. A number of approaches have formulated LMI conditions \cite{Kaszkurewicz:2000,Barabanov:2002,Chu_bounds_1999} guaranteeing Lyapunov stability of a particular equilibrium. Recently, incremental quadratic constraints  \cite{Zames:1966}  have been recently applied to (non-recurrent) neural networks to develop the tightest bounds on the Lipschitz constant known to date \cite{Fazlyab:2019}. %\tb{While not directly applicable to the subject matter of this paper}, this line of research shows that employing incremental quadratic constraintsto describe the possible behavior of a neural network's non-linearities can effectively decrease the conservatism of stability criteria. In particular, 

 {\it Contributions:} In this letter we propose a new convex parameterization of RNNs satisfying stability and robustness conditions. By treating RNNs as linear systems in feedback with a slope-restricted, element-wise nonlinearity, we can apply methods from robust control to develop stability conditions that are less conservative than prior methods. The proposed model set contains all previously published sets of stable RNNs, and all stable linear time-invariant (LTI) systems.
	Using implicit parameterizations with incremental quadratic constraints, we construct a set of models that is jointly convex in the model parameters, stability certificate and the multipliers required by the incremental quadratic constraint approach. Joint convexity in all parameters simplifies the training of stable models as constraints can be easily dealt with using penalty, barrier or projected gradient methods. 

\noindent {\it Notation.} We use $ \N,\R $ to denote the set of natural and real numbers, respectively. The set of all one-side sequences $ x:\N\rightarrow\R^n $ is denoted by $ \ell_{2e}^n $. Superscript $ n $ is omitted when it is clear from the context.  For $x\in\ell_{2e}^n$, $x_t\in\R^n$ is the value of the sequence $ x $ at time $t\in \N$. The notation $ |\cdot|:\R^n\rightarrow\R $ denotes the standard 2-norm. The subset $ \ell_2\subset \ell_{2e} $ consists of all square-summable sequences, i.e., $ x\in\ell_2 $ if and only if the $ \ell_2 $ norm $ \|x\|:=\sqrt{\sum_{t=0}^{\infty}|x_t|^2} $ is finite. Given a sequence $ x\in\ell_{2e} $, the $ \ell_2 $ norm of its truncation over $ [0,T] $ with $ T\in\N $ is written as $ \|x\|_T:=\sqrt{\sum_{t=0}^{T}|x_t|^2} $. 
%Given a piecewise differentiable vector function $ f:\R^n\rightarrow\R^m $, we use $ D^+f(x;v) $ to  denote the one-side directional derivative of $ f(\cdot) $ at $ x $ in the direction $ v $, i.e. $ D^+f(x;v):=\lim\limits_{s\rightarrow 0^+}[f(x+sv)-f(x)]/s $. If $ f $ is differentiable, then $ D^+f(x;v)=\frac{\partial f}{\partial x}v $. 
For matrices $A$, we use $A\succ 0$ and $A\succeq0$ to mean $A$ is positive definite or positive semi-definite respectively and $A\succ B$ and $A\succeq0$ to mean $A-B\succ0$ and $A-B\succeq0$ respectively. The set of diagonal, positive definite matrices is denoted $\mathbb{D}_+$. 

%% file: Problem.tex
\section{Problem Formulation}

 We are interested in learning nonlinear state space models:
\begin{gather}
x_{t+1}=f_\theta(x_t,u_t) \label{eq:rnn-ss}, \\
y_t=g_\theta(x_t,u_t) \label{eq:rnn-output},
\end{gather}
where $ x_t \in \mathbb{R}^n $ is the state, $u_t \in \mathbb{R}^{m}$ is a known input and $y_t \in \mathbb{R}^{p}$ is the output.
The functions $ f_\theta,g_\theta $ are parametrized by $\theta\in\Theta\subseteq\R^N$ and will be defined later.
Given initial condition $x_0 = a$, the dynamical system \eqref{eq:rnn-ss}, \eqref{eq:rnn-output} can be provides a sequence-to-sequence mapping $ S_{a}:\ell_{2e}^m\mapsto \ell_{2e}^p $.

% whose model can be described by a dynamical system with finite-dimensional state $ x_t \in \mathbb{R}^n $ and initial state $ x_0=a $, driven by a known input $u_t \in \mathbb{R}^{m}$ and producing an output $y_t \in \mathbb{R}^{p}$.
% parametrized by $\theta\in\Theta\subseteq\R^N $
 
%where functions $ f_\theta,g_\theta $ will be defined later. 

%\remove{It is often desired to learn RNNs which have predictable responses to a wide variety of inputs. One direct approach is to search for RNNs from a model set with stability and robustness guarantees.} \tb{One Sentence.} First, we introduce some stability definitions.

\begin{definition}
	The system \eqref{eq:rnn-ss}, \eqref{eq:rnn-output} is termed \textit{incrementally $\ell_2$ stable} if for any two initial conditions $ a $ and $ b $, given the same input sequence $u$, the corresponding output trajectories $ y^a $ and $ y^b $ satisfy $ y^a-y^b\in\ell_2^p $. 
\end{definition}
This definition implies that initial conditions are forgotten, however, the outputs can still be sensitive to small perturbations in the input. In such cases, it is natural to measure system robustness in terms of the incremental $\ell_2$-gain.
\begin{definition}
	The system \eqref{eq:rnn-ss}, \eqref{eq:rnn-output} is said to have an \emph{incremental $\ell_2$-gain} bound of $\gamma$ if for all pairs of solutions with initial conditions $ a,b\in\R^n $ and input sequences $ u^a,u^b\in \ell_{2e}^m $, the output sequences $ y^a,y^b\in\ell_{2e}^p $ satisfy
	\begin{equation}\label{eq:L2-gain}
	\left\|y^a-y^b\right\|_T^2\leq \gamma^2\left\|u^a-u^b\right\|_T^2+d(a,b),\quad\forall T\in\N,
	\end{equation}
	for some $d(a,b) \geq0$ with $d(a,a) = 0$.
\end{definition}
Note that the above definition implies  incremental $ \ell_2 $ stability since $ \|y^a-y^b\|_T^2\leq d(a,b) $ for all $ T\in \N $ when $ u^a=u^b $. It also shows that all operators defined by \eqref{eq:rnn-ss} and \eqref{eq:rnn-output} are Lipschitz continuous with Lipschitz constant $\gamma$, i.e. for any $ a\in \R^n $ and all $ T\in \N $
\begin{equation}
\|S_a(u) - S_a(v)\|_T \leq  \gamma \|u - v\|_T,\quad \forall u,v \in \ell_{2e}^m.  
\end{equation}

The goal of this work is to construct a rich parametrization of the functions $f_\theta$ and $g_\theta$ in \eqref{eq:rnn-ss}, \eqref{eq:rnn-output}, with robustness guarantees. We focus on two robustness guarantees in this work:
\begin{enumerate}
	\item A model set parametrized by $\Theta_* \subset R^N$ is robust if for all $\theta \in \Theta_*$ the system has finite incremental $\ell_2$-gain.
	\item A model set parameterized by $\Theta_\gamma \subset R^N$ is $\gamma$-robust if for all $\theta \in \Theta_\gamma$ the system has an incremental $ \ell_2 $-gain bound of $ \gamma $.
\end{enumerate}

%% file: IQC_RNN.tex
\section{Robust RNNs}

\subsection{Model Structure}

%\tb{Our work can be applied to the models with multi-layer network. To streamline the presentation, we will focus on the case with one-layer network. }

%\tb{Our work can be applied to multi-layer networks. For brevity and clarity, however, we will focus on the single layer case. The approach to the multi-layer is the same as in \cite{Revay:2019}.}

\begin{figure}
	\centering
	\includegraphics[width=0.45\linewidth]{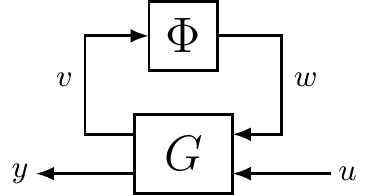}
	\caption{\label{fig:feedback} Feedback interconnection for RNNs.}
\end{figure} 

We parameterize the functions \eqref{eq:rnn-ss}, \eqref{eq:rnn-output} as a feedback interconnection between a linear system $G$ and a static, memoryless nonlinear operator $\Phi$:
%As shown in Fig.~\ref{fig:feedback}, we treat the RNN as a feedback interconnection of a linear system $ G $ and a static, memoryless nonlinear operator $ \Phi $
\begin{gather} \label{eq:G-op} 
G\;\begin{cases}
x_{t+1} = \bar{F} x_t + \bar{B}_1 w_t + \bar{B}_2 u_t \\
y_t = C_1 x_t + D_{11} w_t + D_{12} u_t \\
v_t = \bar{C}_2 x_t + \bar{b} + \bar{D}_{22} u_t
\end{cases}, \\ 
w=\Phi(v), \label{eq:Phi-op}
\end{gather}
where $ \Phi(v)=[\phi(v_1)\ \cdots\ \phi(v_q)]^\top $ with $ v_i $ as the $ i $th component of the $ v \in \ell_{2e}^q$. This feedback interconnection is shown in Fig. \ref{fig:feedback}. We assume that the slope of $\phi$ is restricted to the interval $[0, \beta]$:
\begin{equation} \label{eq:slope_restricted}
0 \leq \frac{\phi(y) - \phi(x)}{y - x} \leq \beta, \quad \forall x,y \in \mathbb{R}, ~ x\neq y.
\end{equation}

In the neural network literature, such functions are referred to as ``activation functions'', and common choices (e.g. tanh, ReLU, sigmoid) are slope restricted \cite{goodfellow2016deep}. 

The proposed model structure is highly expressive and contains many commonly used model structures. 
For instance, LTI systems are obtained when $\bar{B}_1 = 0$ and $D_{11}=0$.
RNNs  of the form \cite{Elman:1990}:
\begin{gather}
x_{t+1}=\mathcal{B}_1\Phi(\mathcal{A}x_t+\mathcal{B}u_t+b), \label{eq:mdl-explicit}\\ 
y_{t}=\mathcal{C}x_t+\mathcal{D}u_t, \label{eq:elman_output}
\end{gather}
are obtained with the choice $\bar{F}=0$, $\bar{B}_1= \mathcal{B}_1$, $\bar{B}_2 = 0$, $C_1=\mathcal{C}$, $D_{11}=0$, $D_{12} = \mathcal{D}$, $\bar{C}_2 = \mathcal{A}$, $\bar{D}_{22} =\mathcal{B}$ and $\bar{b} = b$. This implies \eqref{eq:G-op}, \eqref{eq:Phi-op} is a universal approximator for dynamical systems over bounded domains as $q\rightarrow\infty$ \cite{pinkus1999approximation}.

%he linear dynamics $F$ allows for the representation of skip/residual connections and linear systems.

%For two RNN parameters $ \theta_1,\theta_2 $, $ \theta_1\sim\theta_2 $ means their corresponding RNNs are input/output equivalent. For two model sets $ \Theta_1,\Theta_2 $, $ \Theta_1\subset\Theta_2 $ implies that for any $ \theta_1\in\Theta_1 $, there exists a $ \theta_2\in\Theta_2 $ such that $ \theta_1\sim\theta_2 $. Note that $ \Theta_e\subset \Theta_i $ as the implicit and explicit models are input/output equivalent if $ E=B_1 $, $ F=0 $, $ B_2=0 $, $ C_1=C $, $ D_{11}=0 $, $ D_{12}=D $, $ C_2=A $ and $ D_{22}=B $.

%The implicit model \eqref{eq:Phi-op}, \eqref{eq:G-op} is called a Robust RNN if its incremental $ \ell_2 $-gain is bounded by some constant $ \gamma $. 
%To characterize the set of Robust RNNs, we first derive the associated differential dynamics, which can also be expressed as a feedback interconnection of a linear system $ \delta G $ and  a differential operator $ \delta \Phi $:
%\begin{align}
%\delta G:\; &
%\begin{cases}
%E\delta_{x_{t+1}}= {F} \delta_{x_t} + {B}_1 \delta_{w_t} + {B}_2 \delta_{u_t} \\
%\delta_{y_t} = C_1 \delta_{x_t} + D_{11} \delta_{w_t} + D_{12} \delta_{u_t} \\
%\Lambda \delta_{v_t} = {C}_2 \delta_{x_t} + {D}_{22} \delta_{u_t}
%\end{cases} \\
%\delta \Phi:\; &\quad \delta_{w_t}=D^+\Phi(v_t;\delta_{v_t}).\label{eq:diff_Phi}
%\end{align}

Even for linear systems, the set of robust or $\gamma$-robust models is non-convex.
Constructing a set of parameters for which  \eqref{eq:G-op}, \eqref{eq:Phi-op} is robust or $\gamma$-robust is further complicated by presence of the nonlinear activation function in $ \Phi $. We will simplify the analysis by replacing $ \Phi $ with incremental quadratic constraints.

%Proving that the model is a robust RNN is complicated by presence of the nonlinear activation function in $ \Phi $. We will simplify the analysis by replacing $ \Phi $ with incremental integral quadratic constraints (incremental IQCs).

\subsection{Description of $\Phi$ by Incremental Quadratic Constraints}
%\remove{An incremental IQC for the operator $ \Phi $ can be viewed as an IQC (\cite{Megretski:1997}) for the differential operator $ \delta \Phi $ \cite{Wang:2019}.}

%We will now construct an incremental quadratic constraint corresponding to \eqref{eq:slope_restricted}. 
%A simple incremental quadratic constraint can be constructed from \eqref{eq:slope_restricted}.
%A simple quadratic constraint 
Multiplying \eqref{eq:slope_restricted} through by $(y-x)^2$, and combining the two inequalities, we get:
\begin{equation}\label{eq:slope_resiction_constraint}
\bmat{y-x \\ \phi(y) - \phi(x)}^\top \bmat{0 & \beta\\\beta & -2}\bmat{y-x \\ \phi(y) - \phi(x)} \geq 0.
\end{equation}
For $v^a, v^b \in \R^q$ and $w^a = \Phi(v^a)$, $w^b = \Phi(v^b)$, \eqref{eq:slope_resiction_constraint} holds for each element with $y=v^a_i$ and $x=v^b_i$. 
%\begin{gather}
%\eqref{eq:slope_restricted} \implies \left(\frac{w^a_i - w^b_i}{v^a_i - v^b_i} \right)  \left(\frac{w^a_i - w^b_i}{v^a_i - v^b_i} -\beta \right) \leq 0 \nonumber\\
%\implies \left(w^a_i - w^b_i\right)(w^a_i - w^b_i - \beta (v^a_i - v^b_i)^2) \leq 0, \label{eq:element_slope_restricted}
%\end{gather}
%where $v_i$ is the $i$'th component of $v$.
% As this holds for each element of the activation function, it holds for a conic combination, so for $\lambda_i\geq0$ 
Taking a conic combination of these constraints with multipliers $\lambda_i>0$, we get:
\begin{gather}  \label{eq:IQC}
%\sum_{i=1}^q 2\lambda_i\left(w^a_i - w^b_i\right)(w^a_i - w^b_i - \beta (v^a_i - v^b_i)^2) \leq 0,\nonumber\\
 \begin{bmatrix}
v^a_t - v^b_t \\ w^a_t - w^b_t
\end{bmatrix}^\top
\underset{M(\Lambda)}{\underbrace{\begin{bmatrix} 0 & \beta \Lambda \\ \beta \Lambda & -2\Lambda
		\end{bmatrix}}}
\begin{bmatrix}
v^a_t - v^b_t \\ w^a_t - w^b_t
\end{bmatrix}\geq 0,
\end{gather}
where $\Lambda = \mathrm{diag}(\lambda_1,...,\lambda_q)$.

\subsection{Convex Parametrization of Robust RNNs}

Corresponding to the linear system \eqref{eq:G-op}, we introduce the following implicit, redundant parametrization:
%The set of linear systems \eqref{eq:G-op} can also be parametrized by the following implicit, redundant parametrization:
\begin{equation} \label{eq:G-op-implicit} 
G\;\begin{cases}
E x_{t+1} = F x_t + B_1 w_t + B_2 u_t \\
y_t = C_1 x_t + D_{11} w_t + D_{12} u_t \\
\Lambda v_t = C_2 x_t + b + D_{22} u_t
\end{cases} 
\end{equation} 
where $ \theta=(E, F,B_1,B_2, C_1,D_{11},D_{12}, \Lambda,C_2, b, D_{22}) $ are the model parameters with $E$ invertible and $\Lambda \in \mathbb{D}_+$ is the incremental quadratic constraint multiplier from \eqref{eq:IQC}. The explicit system \eqref{eq:G-op} can be easily constructed from \eqref{eq:G-op-implicit} by inverting $E$ and $\Lambda$.
While the parameters $E$ and $\Lambda$ do not improve model expressiveness, the extra degrees of freedom will allow us to formulate sets of robust models that are jointly convex in the model parameters, stability certificate and multipliers. 

To construct the set of stable robust models, we introduce the following convex constraint:
\begin{multline} \label{eq:bound_l2_gain_lmi}
\begin{bmatrix}
E + E^\top - P  & - \beta {C}^\top_2 \\
- \beta {C}_2 &  2 \Lambda
\end{bmatrix} - 
\begin{bmatrix}
{F}^\top \\ {B}_1^\top 
\end{bmatrix}P^{-1}
\begin{bmatrix}
{F}^\top \\ {B}_1^\top 
\end{bmatrix}^\top   
\succ 0,
\end{multline}
The set of Robust RNNs is then given by:
\[
\Theta_{*}:=\{\theta:\exists P\succ 0,\,\Lambda\in\mathbb{D}_+ \;\mathrm{s.t.}\;\eqref{eq:bound_l2_gain_lmi}\}.
\]
Since $P\succ0$, \eqref{eq:bound_l2_gain_lmi} and \eqref{eq:dl2_gain_bound_lmi} imply that $E+E^\top \succ 0$ which means that $E$ is invertible. 

%\begin{multline} \label{eq:bound_l2_gain_lmi}
%	\begin{bmatrix}
%		E + E^\top - P  & - \beta {C}^\top_2 \\
%		- \beta {C}_2 &  2 T(\lambda) 
%	\end{bmatrix} - 
%	\begin{bmatrix}
%		{F}^\top \\ {B}_1^\top 
%	\end{bmatrix}P^{-1}
%	\begin{bmatrix}
%		{F}^\top \\ {B}_1^\top 
%	\end{bmatrix}^\top   
%	\succ 0
%\end{multline}
%\begin{multline} \label{eq:bound_l2_gain_lmi}
%\begin{bmatrix}
%E + E^\top - P  & 0 \\
%0 & 0 
%\end{bmatrix} - 
%\begin{bmatrix}
%F^\top \\ B_1^\top 
%\end{bmatrix}P^{-1}
%\begin{bmatrix}
%F^\top \\ B_1^\top 
%\end{bmatrix}^\top   \\
%-
%\begin{bmatrix}
%\widehat{C}_2^\top \\ \widehat{D}_{21}^\top
%\end{bmatrix}
%M(\lambda)
%\begin{bmatrix}
%\widehat{C}_2^\top \\ \widehat{D}_{21}^\top
%\end{bmatrix}^\top
%\succ 0
%\end{multline}
%The set of Robust RNNs with finite incremental $ \ell_2 $-gain 

%\[
%\Theta_{*}:=\{\theta:\exists P\succ 0,\,\Lambda\in\mathbb{D}_+ \;\mathrm{s.t.}\;E+E^\top\succ 0,\;\eqref{eq:bound_l2_gain_lmi}\}.
%\]

%To obtain a $\gamma$-Robust RNN with a incremental $ \ell_2 $-gain bound of $ \gamma $, we propose to use the following constraint:
%To obtain a $\gamma$-Robust RNN, we propose to use the following constraint:
To construct a set of $\gamma$-robust models, we propose the following convex constraint:
\begin{multline} \label{eq:dl2_gain_bound_lmi}
\begin{bmatrix}
	E + E^\top - P  & -\beta {C}_2^\top &0\\
	-\beta {C}_2 & 2\Lambda & -\beta {D}_{22}^\top \\
	0 & -\beta {D}_{22} & \gamma I
\end{bmatrix} \\ - 
\begin{bmatrix}
	F^\top \\ B_1^\top \\ B_2^\top
\end{bmatrix}P^{-1}
\begin{bmatrix}
F^\top \\ B_1^\top \\ B_2^\top
\end{bmatrix}^\top 
- \frac{1}{\gamma}
\begin{bmatrix}
	C_{1}^\top \\ D_{11}^\top \\D_{12}^\top
\end{bmatrix}
\begin{bmatrix}
C_{1}^\top \\ D_{11}^\top \\D_{12}^\top
\end{bmatrix}^\top \succ 0,
\end{multline}
% Note that if the LMI condition \eqref{eq:bound_l2_gain_lmi} is satisfied, there exists a sufficiently large $ \gamma $ such that \eqref{eq:dl2_gain_bound_lmi} holds for any choice of $ B_2 $, $ C_1 $, $ D_{11} $, $ D_{12} $ and $ D_{22}$.
% We define the set of Robust RNNs with incremental $ \ell_2 $-gain bound of $ \gamma $ as follows
 The set of $\gamma$-robust RNNs is then given by:
\[
\Theta_{\gamma}:=\{\theta:\exists P\succ 0,\,\Lambda\in\mathbb{D}_+ \;\mathrm{s.t.}\;\eqref{eq:dl2_gain_bound_lmi}\}.
\]
%\[
%\Theta_{\gamma}:=\{\theta:\exists P\succ 0,\,\Lambda\in\mathbb{D}_+ \;\mathrm{s.t.}\;E+E^\top\succ 0,\;\eqref{eq:dl2_gain_bound_lmi}\}.
%\]

Note that \eqref{eq:bound_l2_gain_lmi} and \eqref{eq:dl2_gain_bound_lmi} are jointly convex in the model parameters, stability certificate, multipliers $\Lambda$ and the incremental $\ell_2$ gain bound $\gamma$.

%\begin{remark}\label{rmk:1}
%	Note that the LMI \eqref{eq:dl2_gain_bound_lmi} is not jointly convex in the weights $C_2$, $D_{22}$ and the $ \delta $-IQC multipliers $ \lambda $. We will see that fixing $C_2$ and $D_{22}$ and optimizing over the remaining parameters still leads to highly expressive models. 
%	Fixing $C_2 = I$ and $D_{22} = 0$, the model set still contains the set of ci-RNNs, see Theorem \ref{Thm:cirnn} below. Alternatively, expressibility can be improved at the cost of computational complexity by fixing $C_2$ and $D_{22}$ as random wide layers (i.e. with large $q$). This is similar to the approach taken in the echo-state network \cite{Jaeger:2003}.
%\end{remark}

%Note that $ \Theta_{\gamma_1}\subset\Theta_{\gamma_2} $ if $ \gamma_1<\gamma_2 $. 
%An alternate approach is the D-K iteration which searches for the weights $C_2$, $D_{22}$ and $ \delta $-IQC multipliers $ \lambda $ alternately. 
%We will see that this approach still allows for highly expressive models.
\begin{theorem}\label{thm:1}
	Suppose that $ \theta\in\Theta_{\gamma} $, then the Robust RNN \eqref{eq:G-op}, \eqref{eq:Phi-op} has a incremental $ \ell_2 $-gain bound of $ \gamma $.
\end{theorem}

\begin{proof}
	Consider two solutions $x^a, x^b \in \ell_{2e}^n$ and outputs $y^a, y^b \in \ell_{2e}^p$ to the system \eqref{eq:rnn-ss}, \eqref{eq:rnn-output} with initial conditions $a, b \in \mathbb{R}^n$ and inputs $u^a, u^b \in \ell_{2e}^m$. Let $\Delta u = u^a - u^b$, $\Delta x = x^a - x^b$, $\Delta v = v^a - v^b$, $\Delta w = w^a - w^b$ and $\Delta y = y^a - y^b$.
	
To establish the incremental $ \ell_2 $-gain bound, we first left and right multiply \eqref{eq:dl2_gain_bound_lmi} by the vectors $\bmat{\Delta x_t^\top, \Delta w_t^\top, \Delta u_t^\top}$ and $\bmat{\Delta{x_t}^\top, \Delta{w_t}^\top, \Delta{u_t}^\top}^\top$. Applying the bound $-E^\top P^{-1}E \preceq  P-E - E^\top$ \cite{Tobenkin:2017} and introducing the storage function $V_t = \Delta x_t^\top E^\top P^{-1}E\Delta x_t$ gives
$$
V_{t+1} - V_t < \gamma|\Delta u_t|^2-\frac{1}{\gamma}|\Delta y_t|^2 -\begin{bmatrix}
\Delta v_t \\ \Delta w_t
\end{bmatrix}^\top
%\begin{bmatrix} 0 & \beta \Lambda \\ \beta \Lambda & -2\Lambda
%		\end{bmatrix}
M(\Lambda)
\begin{bmatrix}
\Delta v_t \\ \Delta w_t
\end{bmatrix}
$$
for $\Delta_x\ne 0$.  Using \eqref{eq:IQC} and summing over $[0,T]$ gives
 \begin{equation*}
 \begin{split}
 V_{T}-V_0 &<  \gamma \|\Delta {u}\|_T^2 - \frac{1}{\gamma}\|\Delta {y}\|_T^2
 \end{split},
 \end{equation*}
%\begin{equation*}
%\begin{split}
%V_{T}-V_0 &\leq  \gamma \|\Delta {u}\|_T^2 - \frac{1}{\gamma}\|\Delta {y}\|_T^2 - \sum_{t=0}^{T}\begin{bmatrix}
%\Delta {v_t} \\ \Delta {w_t}
%\end{bmatrix}^\top
%M
%\begin{bmatrix}
%\Delta {v_t} \\ \Delta {w_t}
%\end{bmatrix}.
%\end{split}
%\end{equation*}
%where $V_t = \Delta {x_t}^\top E^\top P^{-1} E \Delta {x_t}$.
for $\Delta_x\ne 0$, so the incremental $ \ell_2 $-gain condition \eqref{eq:L2-gain} follows with $ d(a, b)=\gamma V_0 $.
\end{proof}

\begin{theorem}
	Suppose that $ \theta\in\Theta_{*} $, then the Robust RNN \eqref{eq:G-op}, \eqref{eq:Phi-op} has a finite incremental $ \ell_2 $-gain.
\end{theorem}
\begin{proof}
Note that if the LMI condition \eqref{eq:bound_l2_gain_lmi} is satisfied, there exists a sufficiently large $ \gamma $ such that \eqref{eq:dl2_gain_bound_lmi} holds for any choice of $ B_2 $, $ C_1 $, $ D_{11} $, $ D_{12} $ and $ D_{22}$.
Since \eqref{eq:bound_l2_gain_lmi} implies \eqref{eq:dl2_gain_bound_lmi} for some sufficiently large $ \gamma $, from Theorem~\ref{thm:1} the Robust RNN \eqref{eq:G-op}, \eqref{eq:Phi-op} has a finite incremental $ \ell_2 $-gain bound of $ \gamma $.
\end{proof}
\begin{remark}
	Theorem 1 and Theorem 2 actually imply a stronger form of stability. For $\Delta u=0$, it is straightforward to show from the strict matrix inequalities that $V_{t+1} \le \alpha V_t$ for some $\alpha\in(0,1)$,  which implies that the dynamics are contracting \cite{Lohmiller:1998}.
\end{remark}

%	Note however, that this approach introduces a nonlinear parametrization of the dynamics in \eqref{eq:G-op}.}

%\tb{The proposed LMI conditions are jointly convex in the model parameters: $E$, $P$, $F$, $B_1$, $B_2$, $C_1$, $D_{11},$ $D_{12}$, $\tilde{C}$, $\tilde{D}$, $b$ and the IQC multipliers $\lambda$. An implicit RNN of the form \eqref{eq:Phi-op} , \eqref{eq:G-op} can be recovered by taking $C_2 = T(\lambda)^{-1}\tilde{C}$ and $D_{22} = T(\lambda)^{-1}\tilde{D}$. }

%\begin{remark}
%	We have derived \eqref{eq:dl2_gain_bound_lmi} to verify the differential $\ell_2$-gain bound \eqref{eq:diff_L2_gain}, however, other properties such as differential dissipativity can also be easily incorporated.
%\end{remark}

\subsection{Expressivity of the model set}
To be able to learn models for a wide class of systems, it is beneficial to have as expressive a model set as possible. The main result regarding expressivity is that the Robust RNN set $ \Theta_* $ contains all contracting implicit RNNs (ci-RNNs) \cite{Revay:2019} and stable LTI models. 

%To explain this result, we first introduce the concept of input/output equivalence, i.e., two RNNs with parameter $ \theta_1 $ and $ \theta_2 $ is said to be input/output equivalent (denoted $ \theta_1\sim\theta_2 $) if they admits the same input/output trajectory set. 
%A model set $ \Theta_1 $ is said to be a subset of another model set $ \Theta_2 $ (denoted $ \Theta_1\subseteq\Theta_2 $) if for any $ \theta_1\in\Theta_1 $, there exists $ \theta_2\in\Theta_2 $ such that $ \theta_1\sim\theta_2 $. And $ \Theta_1\subset\Theta_2 $ means $ \Theta_1 $ is a strict subset of $ \Theta_2 $. \remove{only need strict subset?}

%The set of LTI systems can be described by the state-space model:

\begin{theorem}
	The Robust RNN set $ \Theta_{*} $ contains all stable LTI models of the form \begin{equation}\label{eq:LTI}
x_{t+1}= \mathcal{A}x_t + \mathcal{B}u_t,\quad
y_t= \mathcal{C}x_t + \mathcal{D}u_t.
\end{equation}  %i.e. $\Theta_{\mathrm{LTI}} \subset \Theta_*$.
\end{theorem}
\begin{proof}
	A necessary and sufficient condition for stability of \eqref{eq:LTI} is the existence of some $\mathcal{P}\succ0$ such that:
	\begin{equation} \label{eq:explicit_lti_lmi}
	\mathcal{P} - \mathcal{A}^\top \mathcal{P} \mathcal{A} \succ 0.
	\end{equation}
	For any stable LTI system, the implicit RNN with $\theta$ such that $E=P=\mathcal{P}$, $F = \mathcal{P}\mathcal{A}$, $ B_1=0 $, $B_2 = \mathcal{P}\mathcal{B}$, $C = \mathcal{C}$ and $D = \mathcal{D}$, $C_2 = 0$ and $D_{22} = 0$ has the same dynamics and output.
	To see that that $ \theta\in\Theta_* $,
	\begin{equation*}
	\begin{split}
	\eqref{eq:explicit_lti_lmi} \Rightarrow & E + E^\top  - P - F^\top P^{-1}P P^{-1}F \succ 0 \\
%	\Rightarrow & E + E^\top  - P - F^\top P^{-1}F \succ 0 \\
	\Rightarrow & \bmat{E + E^\top - P - F^\top P^{-1} F & 0 \\ 0 & 2\Lambda} \succeq 0\Rightarrow \eqref{eq:bound_l2_gain_lmi}
	\end{split}
	\end{equation*}
	for any $\Lambda \succ 0$. 
%	where $ T\in \mathbb{D}_+ $ satisfies $ T/2\preceq E + E^\top  - P - F^\top P^{-1}F $. 
\end{proof}

\begin{remark}
	Essentially the same proof technique but with the strict Bounded Real Lemma can be used to show that $\Theta_\gamma$ contains all LTI models with an  $H_\infty$ norm of $\gamma$.
\end{remark}

A ci-RNN \cite{Revay:2019} is an implicit model of the form:
\begin{equation}\label{eq:ci-rnn}
	\mathcal{E}z_{t+1}=\Phi(\mathcal{F}z_t+\mathcal{B}u_t+\mathfrak{b}), \quad
	y_t=\mathcal{C}z_t+\mathcal{D}u_t
\end{equation}
such that the following contraction condition holds
\begin{equation}\label{eq:contraction}
	\begin{bmatrix}
		\mathcal{E}+\mathcal{E}^T-\mathcal{P} & \mathcal{F}^T \\
		\mathcal{F} & \mathcal{P}
	\end{bmatrix}\succ 0
\end{equation}
where $ \mathcal{P}\in\mathbb{D}_+ $. 
%We define the convex set of ci-RNNs as
%\begin{equation}
%	\Theta_{\mathrm{ci}}:=\{\theta:\exists \mathcal{P}\in\mathbb{D}_+\quad \mathrm{s.t.}\quad \eqref{eq:contraction}\}.
%\end{equation}
The stable RNN (s-RNN), proposed in \cite{Miller:2019} is contained within the set of ci-RNNs when $\mathcal{E}=I$.

%Note that $ \Theta_{\mathrm{ci}} $ does not contain all stable LTI systems. For example, the system $ x_{t+1}=0.5x_t+u_t,\; y_t=x_t $ cannot be converted into the form \eqref{eq:ci-rnn} via coordinate transformation.
\begin{theorem} \label{Thm:cirnn}
	The Robust RNN set $ \Theta_* $ contains all ci-RNNs. % i.e. $ \Theta_{\mathrm{ci}}\subset \Theta_{*} $.
\end{theorem}
\begin{proof}
For any ci-RNN, there is an implicit RNN with the same dynamics and output with $\theta$ such that
%we first show that $ \theta'\sim \theta $ where $ \theta $ represents the implicit model with 
$ F=0 $, $ E=\mathcal{E}$, $ B_1=I $, $ B_2=0 $, $C_1=\mathcal{C}$, $D_{11}=0$, $ D_{12}=\mathcal{D} $, $ \Lambda^{-1}C_2 = \mathcal{F}$, $ \Lambda^{-1}D_{22}=\mathcal{B}$, $\mathfrak{b} = \Lambda^{-1}b$, $\Lambda = P^{-1}$ and $P = \mathcal{P}$. By substituting $ \theta $ into \eqref{eq:G-op} and \eqref{eq:Phi-op}, we recover the dynamics and output of the ci-RNN in \eqref{eq:ci-rnn}.

For this parameter choice, $\theta \in \Theta_{*}$. To see this: 
%\begin{align*}
%\eqref{eq:contraction} \implies & \bmat{E + E^\top - P & C_2^\top \Lambda^{-1}\\ \Lambda^{-1} C_2 & P} \succ0 \\
%\implies & \bmat{E + E^\top - P & C_2^\top \\ C_2 & \Lambda P \Lambda} \succ 0 \\
%\implies & \bmat{E + E^\top - P & C_2^\top \\ C_2 & 2\Lambda - P^{-1}} \succ 0 \implies \eqref{eq:bound_l2_gain_lmi}.
%\end{align*}
\begin{align*}
\eqref{eq:contraction} \implies & E + E^\top - P - C_2^\top \Lambda^{-1} P^{-1} \Lambda^{-1} C_2 \succ 0 \\
%\implies & \bmat{E + E^\top - P & C_2^\top \\ C_2 & \Lambda P \Lambda} \succ 0 \\
\implies & \bmat{E + E^\top - P & C_2^\top \\ C_2 & 2\Lambda - P^{-1}} \succ 0 \implies \eqref{eq:bound_l2_gain_lmi}.
\end{align*}
The remaining conditions $P\succ 0$, $\Lambda \in \mathbb{D}_+$ and $E + E^\top \succ 0$ follow by definition.
\end{proof}

%% file: Examples.tex
\section{Numerical Example}
We will compare the proposed Robust RNN with the (Elman) RNN \cite{Elman:1990} described by \eqref{eq:mdl-explicit}, \eqref{eq:elman_output} with $B_1=I$ and Long Short Term Memory (LSTM) \cite{Hochreiter:1997}, which is a widely-used model class that was originally proposed to resolve issues related to stability. In addition, we compare to two previously-published stable model sets, the contracting implicit RNN (ci-RNN) \cite{Revay:2019} and stable RNN (sRNN) \cite{Miller:2019}.  All models have a state dimension of 10 and all models except for the LSTM use a ReLU activation function. The LSTM is described by the following equations:
\begin{gather}
\mathrm{LSTM}\begin{cases}
i_{t+1} = \sigma(W_{xi}x_t + W_{ii}u_{t+1} + b_i), \\
f_{t+1} = \sigma(W_{xf}x_t + W_{if}u_{t+1} + b_f), \\
g_{t+1} = \sigma(W_{xg}x_t + W_{ig}u_{t+1} + b_g), \\
o_{t+1} = \sigma(W_{xo}x_t + W_{io}u_{t+1} + b_o), \\
c_{t+1} = f_{t+1} \odot c_{t} + i_{t+1}\odot g_{t+1}, \\
x_{t+1} = o_{t+1} \odot \tanh(c_{t+1}),
\end{cases}
\end{gather}
where $c_t, x_t\in \R^n$, are the cell state and  hidden state, $u_t\in \R^m$ is the input and $\odot$ is the Hadamard product and $\sigma$ is the sigmoid function. The output is a linear function of the hidden state.

To generate data, we use a simulation of four coupled mass spring dampers. The goal is to identify a mapping from the force on the initial mass to the position of the final mass. Nonlinearity is introduced through the springs' piecewise linear force profile 
\begin{gather} \label{eq:nl_spring}
F_{i}(d) = k_i \Gamma(d), \; \;
\Gamma(d) = \begin{cases}
d + 0.75, 		& -d\leq-1,\\
0.25 d, & -1<d<1,\\
d - 0.75, 		& d\geq1,
\end{cases}
\end{gather}
where $k_i$ is the spring constant for the $i$th spring and $d$ is the displacement between the carts. A schematic is shown in Fig. \ref{fig:msd_schematic}. The masses are $[m_1,...,m_4] = [1/4, 1/3, 5/12, 1/2]$, the linear damping coefficients used are $[c_1,...,c_4] = [1/4, 1/3, 5/12, 1/2]$ and spring constants used in \eqref{eq:nl_spring} are $[k_1,...,k_4] =  [1, 5/6, 2/3, 1/2]$.

\begin{figure}
	\centering
	\includegraphics[width=\linewidth, trim={0cm, 21cm, 0cm, 1.0cm}, clip]{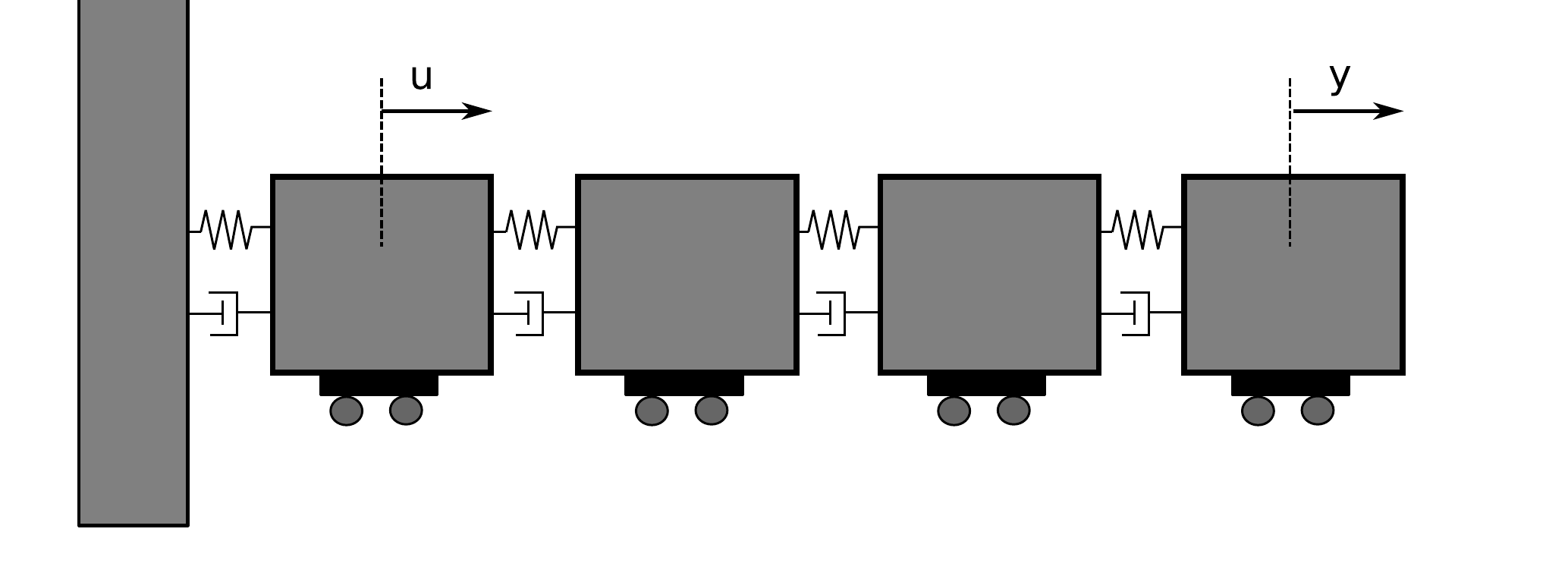}
	\caption{Nonlinear mass spring damper schematic.\label{fig:msd_schematic}}
\end{figure}
%\begin{figure}
%	\centering
%	\includegraphics[width=0.6\linewidth, trim={2cm, 0cm, 3cm, 1cm}, clip]{spring_profile.pdf}
%	\caption{Nonlinear spring profile.\label{fig:nonlinearspringprofile}}
%\end{figure}

%

%\begin{table}
%	\centering
%	\begin{tabular}{|c||c|c|c|c|}
%	\hline  i & 1 & 2 & 3 & 4 \\
%\hline\hline k & 1 & 5/6 & 2/3 & 1/2 \\
%	\hline 	c & 1/4 & 1/3 & 5/12 & 1/2 \\
%	\hline 	m & 1/4 & 1/3 & 5/12 & 1/2 \\\hline 
%	\end{tabular}
%\caption{Parameters for nonlinear mass-spring-damper simulation. \label{tab:sim_parameters}}
%\end{table}

We excite the system with a piecewise-constant input signal that changes value after an interval distributed uniformly in $[0, \tau]$ and takes values that are normally distributed with standard deviation $\sigma_u$. The measurements have Gaussian noise of approximately $30\mathrm{dB}$ added.
To generate data we simulate the system for $T/5$ seconds and sample the system at 5Hz to generate $T$ data points with an input signal characterized by $\tau=20s$ and $\sigma_u=3N$. The training data consists of 100 batches of length $1000$. 
We also generate a validation set with $\tau=20s$, $\sigma_u=3N$ and length $5000$ that is used for early stopping.
To test model performance, we generate test sets of length $1000$ with $\tau=20s$ and varying $\sigma_u $.

\subsection{Training Procedure} \label{sec:training_procedure}
%Fitting Robust RNNs requires a constrained optimization problem to be solved subject to a number of LMI constraints. I.e. we are interested in solving the following optimization problem:

%We fit Robust RNNs by solving the following problem:
%\begin{equation} \label{eq:optimization_problem}
%\min_{\theta\in \Theta_\gamma} ||\tilde{y} - S(\tilde{u})||^2 
%\end{equation}

%Where, $\Theta_\gamma$ is defined in \eqref{eq:dl2_gain_bound_lmi}.

We fit Robust RNNs by optimizing simulation error using stochastic gradient descent and logarithmic barrier functions to ensure strict feasibility of the robustness constraints.
We use the ADAM optimizer \cite{Kingma:2017}  with an initial learning rate of $1\times 10^{-3}$ to optimize the following objective function:
%minimize \eqref{eq:ipm_objective} using the ADAM optimizer with an initial learning rate of $1E-3$.
%The constraint $\theta \in \Theta_{\gamma}$ is enforced using an interior point method where we minimize a series of objective functions of the following form:
\begin{equation*} \label{eq:ipm_objective}
J = ||\tilde{y}^k - S(\tilde{u}^k)||^2 - \sum_i \alpha \log \det (M_i) - \sum_j \alpha \log \lambda_j,
\end{equation*}
where $M_i$ are the LMIs to be satisfied and $\lambda_j$ are the incremental quadratic constraint multipliers and $\tilde{u}^k$ and $\tilde{y}^k$ are the input and output for the $k$th batch. 
%As $\alpha\rightarrow 0$, this approaches the solution of the problem \eqref{eq:optimization_problem}. 
%We minimize \eqref{eq:ipm_objective} using the ADAM optimizer \cite{Kingma:2017} with an initial learning rate of $1E-3$. 
A backtracking line search ensures strict feasibility throughout optimization. After 10 epochs without an improvement in validation performance, we decrease the learning rate by a factor of $0.25$ and decrease $\alpha$ by a factor of $10$. When  $\alpha$ reaches a final value of $1\times10^{-7}$, we finish training.
All code is written using Pytorch 1.60 and run on a standard desktop CPU. The code is available at the following link: \url{https://github.com/imanchester/RobustRNN/}.

%\remove{Model initialization.}
%We initialize all Robust RNNs by first solving a linear subspace identification problem using N4SID to find matrices A,B,C and D and using the resulting matrices to initialize $E^{-1}F = A$, $E^{-1}B_2 = B$, $C_1 = C$ and $D_{12} = D$.

\subsection{ Model Evaluation}
Model quality of fit is measured using normalized simulation error:
\begin{equation}
\text{NSE}  = \frac{||\tilde{y} - y||}{||\tilde{y}||}
\end{equation}
where $y, \tilde{y} \in \ell_2^p$ are the simulated and measured system outputs respectively. 
%In order to study robustness around a certain trajectory $(\bar{x}, \bar{u})$, we will look at the differential $\ell_2$ gain along that trajectory, defined as:
%$$
%\gamma_l = \max_{\delta_u} \frac{|\delta_y|}{|\delta_u|}
%$$
Model robustness is measured by approximately solving:
\begin{equation}
\hat{\gamma} = \max_{u, v} \frac{||S(u) - S(v)||}{||u-v||}, \quad u\neq v. \label{eq:Lip_opt}
\end{equation}
using gradient ascent.
%While solving \eqref{eq:Lip_opt} exactly is complicated by non-convexity, an approximate solution can be found using gradient ascent. 
%An approximate solution to \eqref{eq:Lip_opt} is found using gradient descent.
The value of $\hat{\gamma}$ is a lower bound on the true Lipschitz constant of the model. 
%This provides a measure of worst case sensitivity to input perturbations.

\subsection{Results}
%\remove{The performance of a number of the models on the various test sets and estimates of the Lipschitz constants are shown in Table \ref{tab:Model_Performance}.} 
The validation performance versus number of epochs is shown in Fig. \ref{fig:train_time}. Note that an epoch occurs after one complete pass through the training data. In this case, this corresponds to 100 batches and gradient descent steps.
Each epoch training the Robust RNN takes twice as long as the LSTM due to the evaluation of the logarithmic barrier functions and the backtracking line search, however we will see that the model offers both stability/robustness guarantees and superior generalizability.

Figure \ref{fig:generalization} presents boxplots and a comparison of the medians for the performance of each model for a number of realizations of the input signal with varying $\sigma_u$. In each plot, there is a trough around $\sigma_u=3$ corresponding to the training data distribution. For the LSTM and RNN, the model performance quickly degrades with varying $\sigma_u$. On the other hand, the stable models exhibit a much slower decline in performance. This supports the claim that model stability constraints can improve model generalization. The Robust RNN set $\Theta_{*}$ uniformly outperforms all other models.

%Fig. \ref{fig:median_comparison} shows the medians of the same dataset for the models with and without stability constraints. 
%Comparing the stable models, we can see that the nominal NSE improves as the stability condition used becomes less conservative. Additionally, comparing the unstable models with the stable models, we can see that the stable models have better generalization, signified by the reduced slope after $\sigma_u=3$.

We have also plotted the worst case observed sensitivity versus median nominal test performance ($\sigma_u=3$) in Fig. \ref{fig:nse_vs_lip}. The Robust RNNs show the best trade-off between nominal performance and robustness signified by the fact that they lie further in the lower left corner. For instance if we compare the LSTM with the Robust RNN ($\Theta_*$), we observe similar nominal performance, however the Robust RNN has a much smaller Lipschitz constant. Varying the incremental $\ell_2$ gain allows us to trade off between model performance and robustness. We can also observe in the figure that the guaranteed upper bounds are quite tight to the observed lower bounds on the Lipschitz constant, especially for the set $\Theta_3$.

In Fig. \ref{fig:trajectory_amp10}, we have the model predictions for the RNN, LSTM and Robust RNN for a typical input with $\sigma_u=10$. We can see that even with the larger inputs, the Robust RNN continuous to accurately track the measured data. The predictions of the LSTM and RNN however deviate significantly from measured data for significant periods.

%We can see that for all plots, there is a trough corresponding to where the training data was drawn with $\sigma_u=3$. 
%Note however, that for the LSTM and the RNN, the performance of the models quickly degrades as the amplitude of the input data distribution increases. 
%For the remaining models, however, we can see that this loss in performance is much more gradual. We interpret this as an improvement in model generalizability, supporting the assertion that the Lipschitz constant is a fundamental quantity affecting a models ability to generalize. 
%Examples of the outputs and errors for these models are shown in Figure \ref{fig:example_traj} for $\sigma_u=3$ and for $\sigma_u=10$.
%Fig. \ref{fig:median_comparison} shows the medians of the same dataset for the models with and without stability constraints. 
%Comparing the stable models, we can see that the nominal NSE improves as the stability condition used becomes less conservative. Additionally, comparing the unstable models with the stable models, we can see that the stable models have better generalization, signified by the reduced slope after $\sigma_u=3$.

\begin{figure}
	\centering
	\includegraphics[width=\linewidth, trim={0.0cm, 6.5cm, 1.5cm, 7cm}, clip]{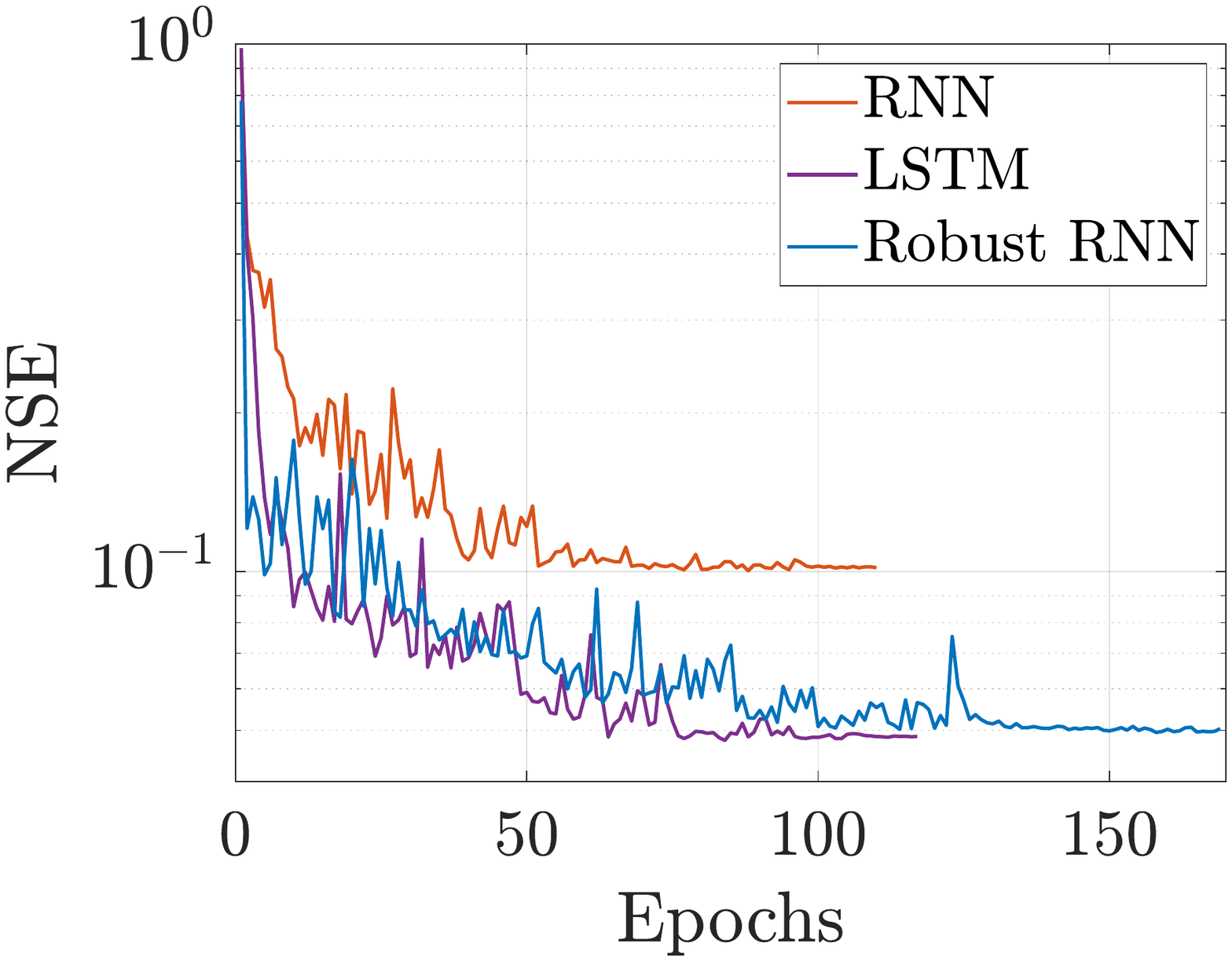}
	\caption{Validation performance versus epochs. The RNN, LSTM and Robust RNN take an average of 10.1, 18.7 and 37.6 seconds per epoch, respectively.\label{fig:train_time}}
\end{figure}

\begin{figure*}
	\centering
\begin{subfigure}{0.32\linewidth}
	\includegraphics[width=\linewidth, trim={2cm, 0.5cm, 3cm, 1.5cm}, clip]{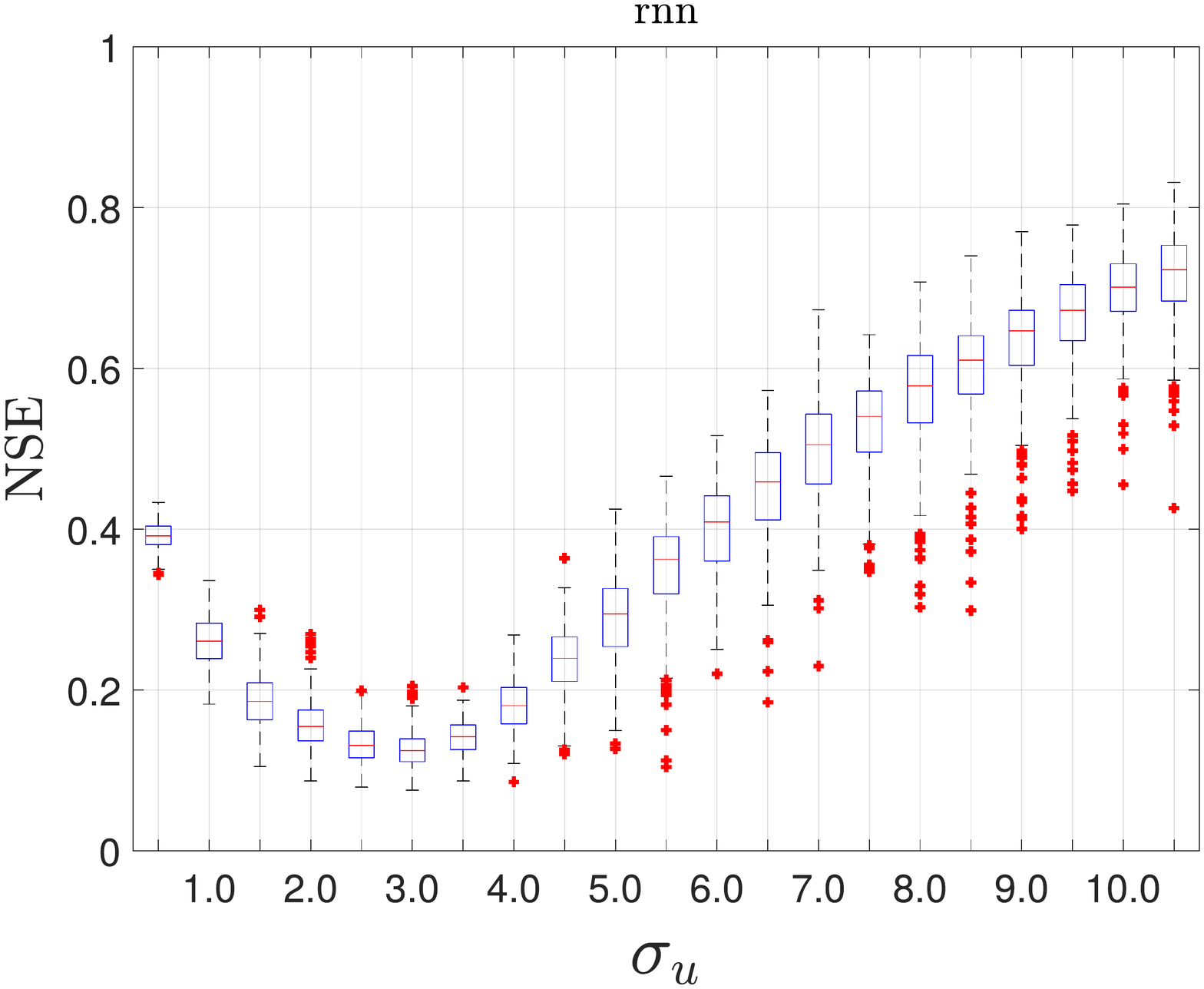}
	\caption{RNN}
\end{subfigure}
\begin{subfigure}{0.32\linewidth}
\includegraphics[width=\linewidth, trim={2cm, 0.5cm, 3cm, 1.5cm}, clip]{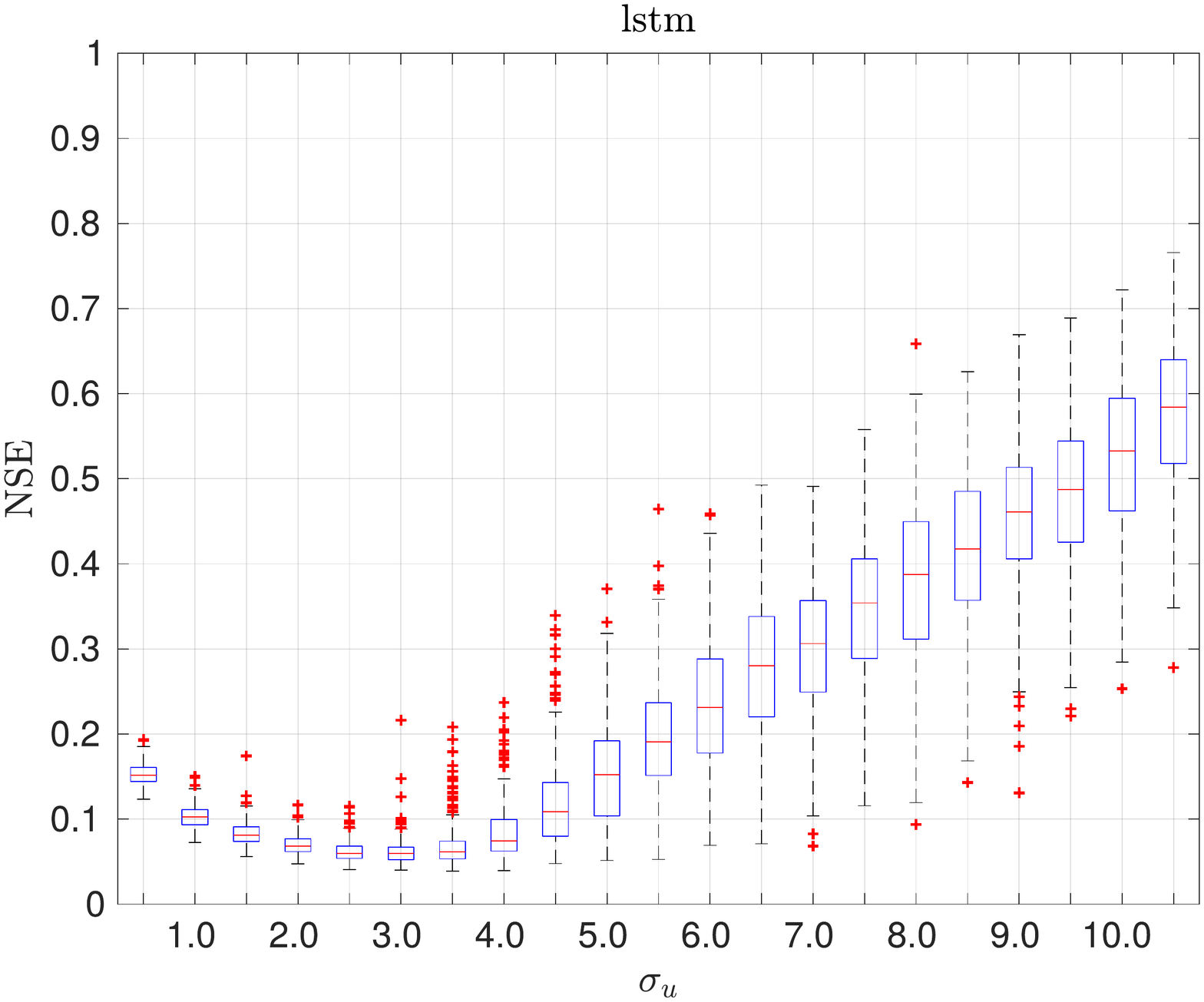}
\caption{LSTM}
\end{subfigure}
\begin{subfigure}{0.32\linewidth}
\includegraphics[width=\linewidth, trim={1.5cm, 0.5cm, 3cm, 1.5cm}, clip]{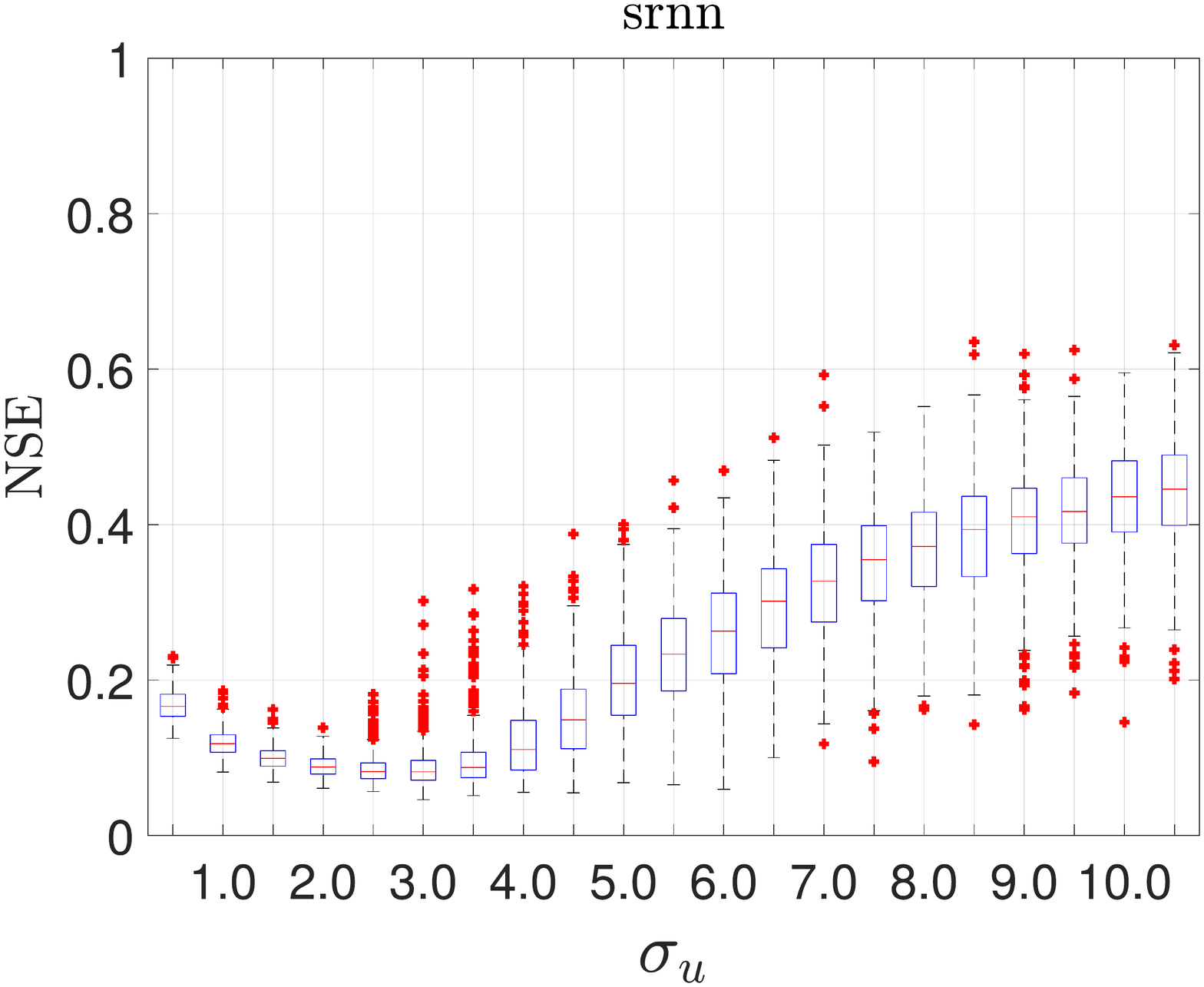}
\caption{s-RNN}
\end{subfigure}
\\
\begin{subfigure}{0.32\linewidth}
	\includegraphics[width=\linewidth, trim={2cm, 0.5cm, 3cm, 1.5cm}, clip]{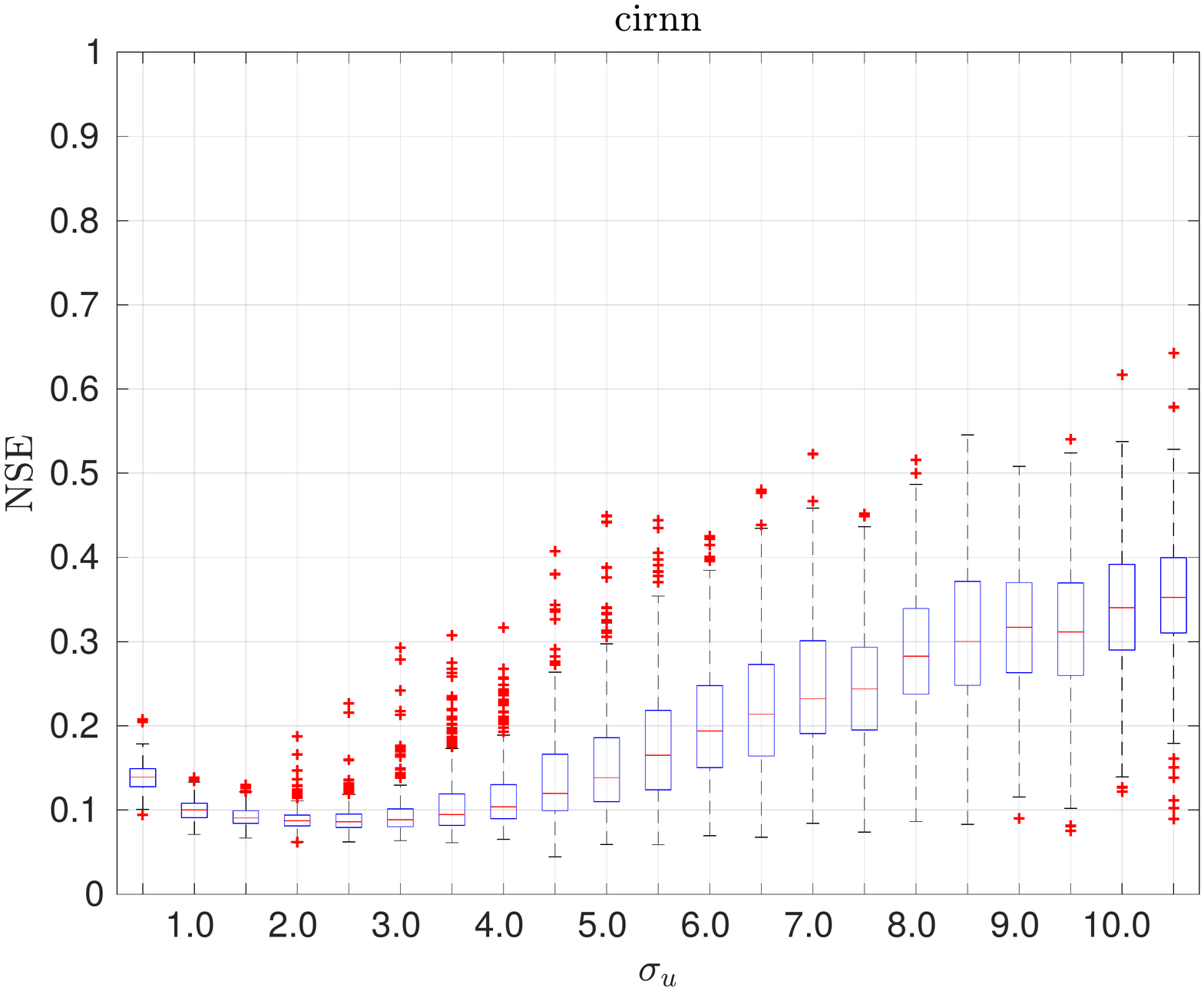}
	\caption{ci-RNN }
\end{subfigure}
\begin{subfigure}{0.32\linewidth}
	\includegraphics[width=\linewidth, trim={2cm, 0.5cm, 3cm, 1.5cm}, clip]{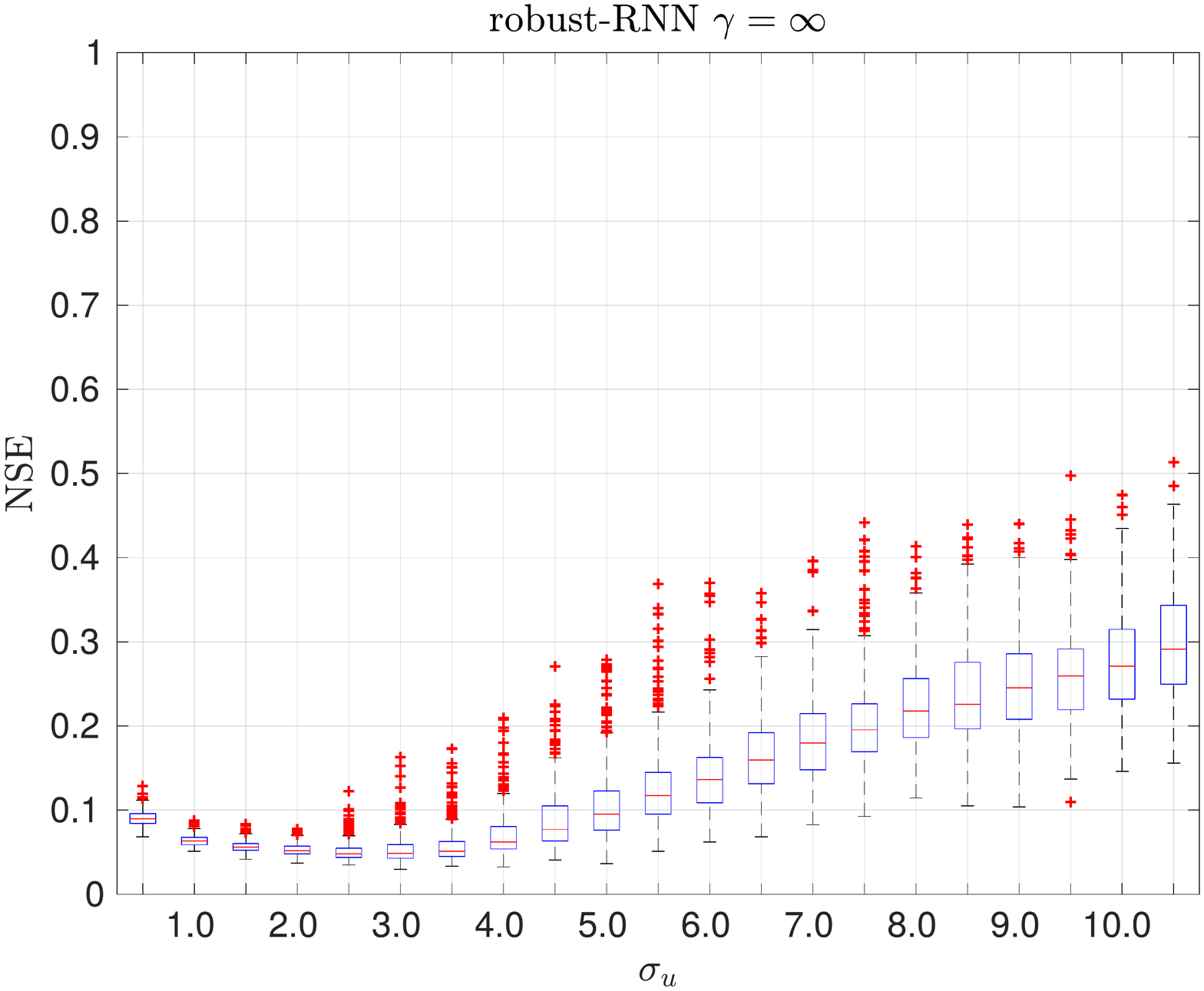}
	\caption{Robust RNN ($\Theta_*$)}
\end{subfigure}
\begin{subfigure}{0.32\linewidth}
	\includegraphics[width=\linewidth, trim={2cm, 0.5cm, 3cm, 0.5cm}, clip]{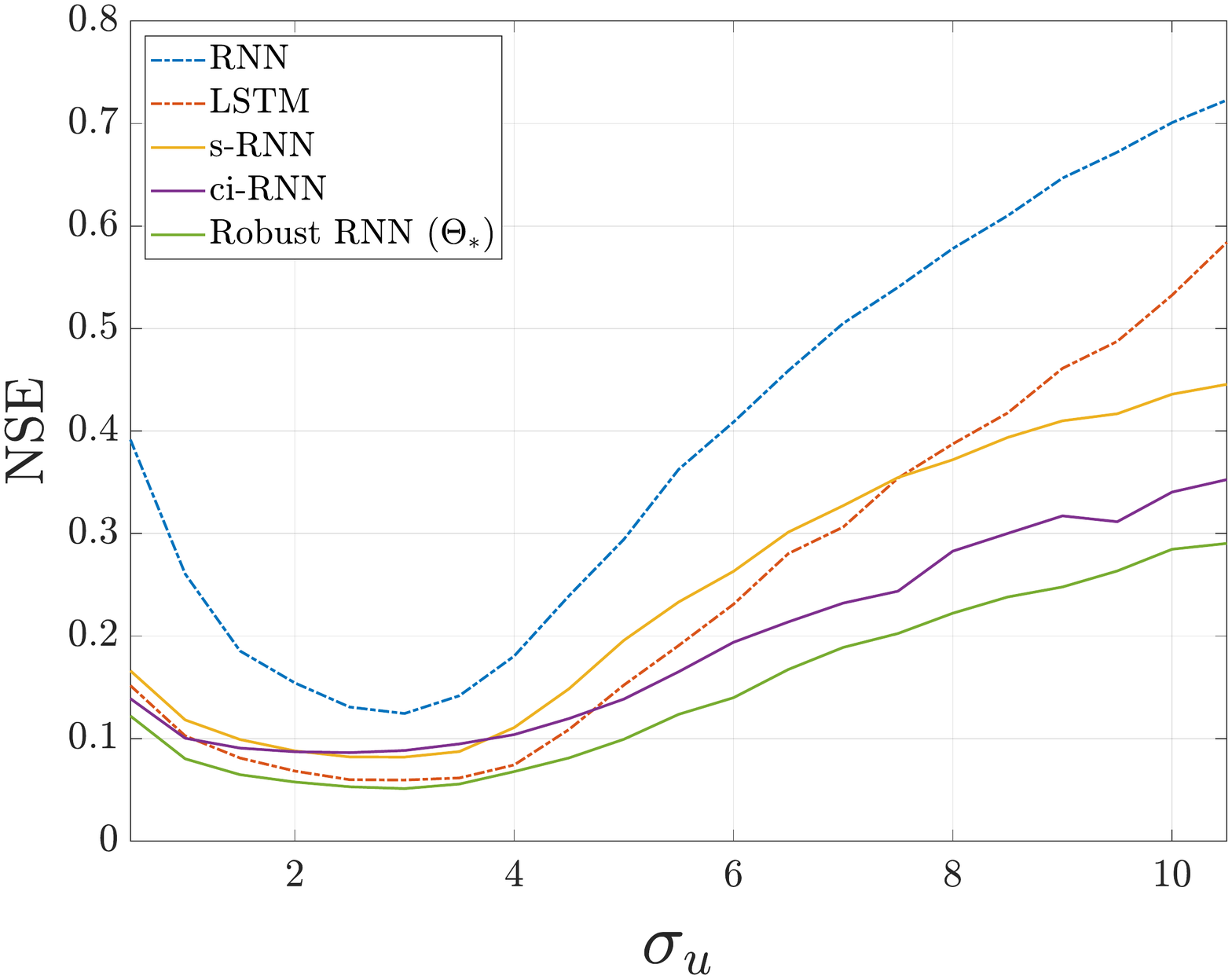}
	\caption{Comparison of median model performance.\label{fig:median_comparison}}
\end{subfigure}
\caption{\label{fig:generalization} Boxplots showing model performance for 300 input realizations for varying $\sigma_u$.  }
\end{figure*}

\begin{figure} 
	\centering
	\includegraphics[width=\linewidth, trim = {1.5cm 0.5cm 3cm 1.3cm}, clip]{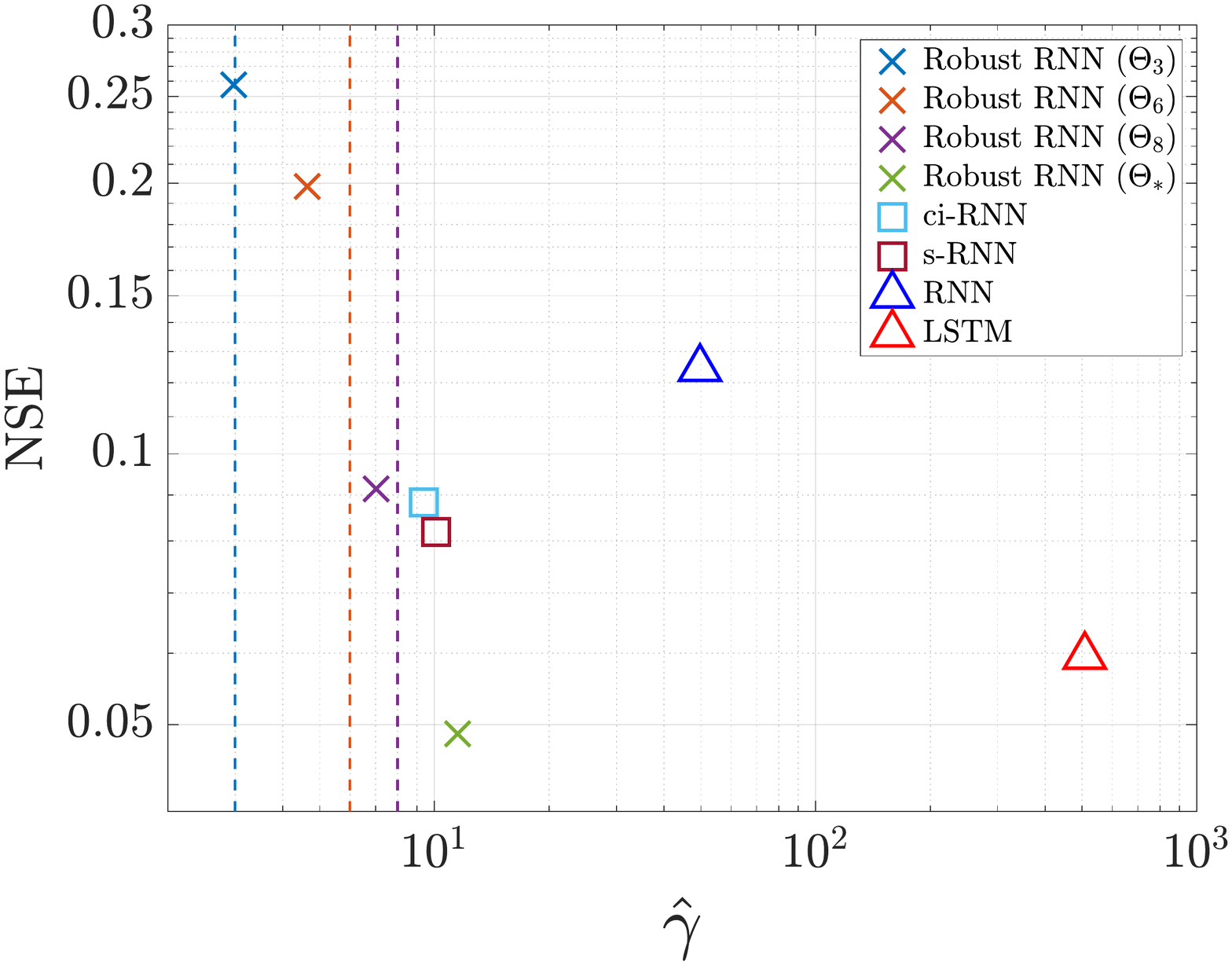}
	\caption{Test Performance ($\sigma_u=3$) on simulated ``mass-spring-damper'' system versus observed worst case sensitivity. The vertical lines represent the theoretical upper bound on the Lipschitz constant. \label{fig:nse_vs_lip}}
\end{figure}
\begin{figure}
	\centering
%	\begin{subfigure}{0.6\linewidth} \centering
%		\includegraphics[width=\linewidth, trim={2cm, 2cm, 3cm, 1cm}, clip]{example_amp3.pdf}
%		\caption{\label{fig:trajectory_amp3}$\sigma_u=3$}
%	\end{subfigure}
%	\begin{subfigure}{0.8\linewidth} \centering
		\includegraphics[width=\linewidth, trim={2cm, 2cm, 3cm, 1.3cm}, clip]{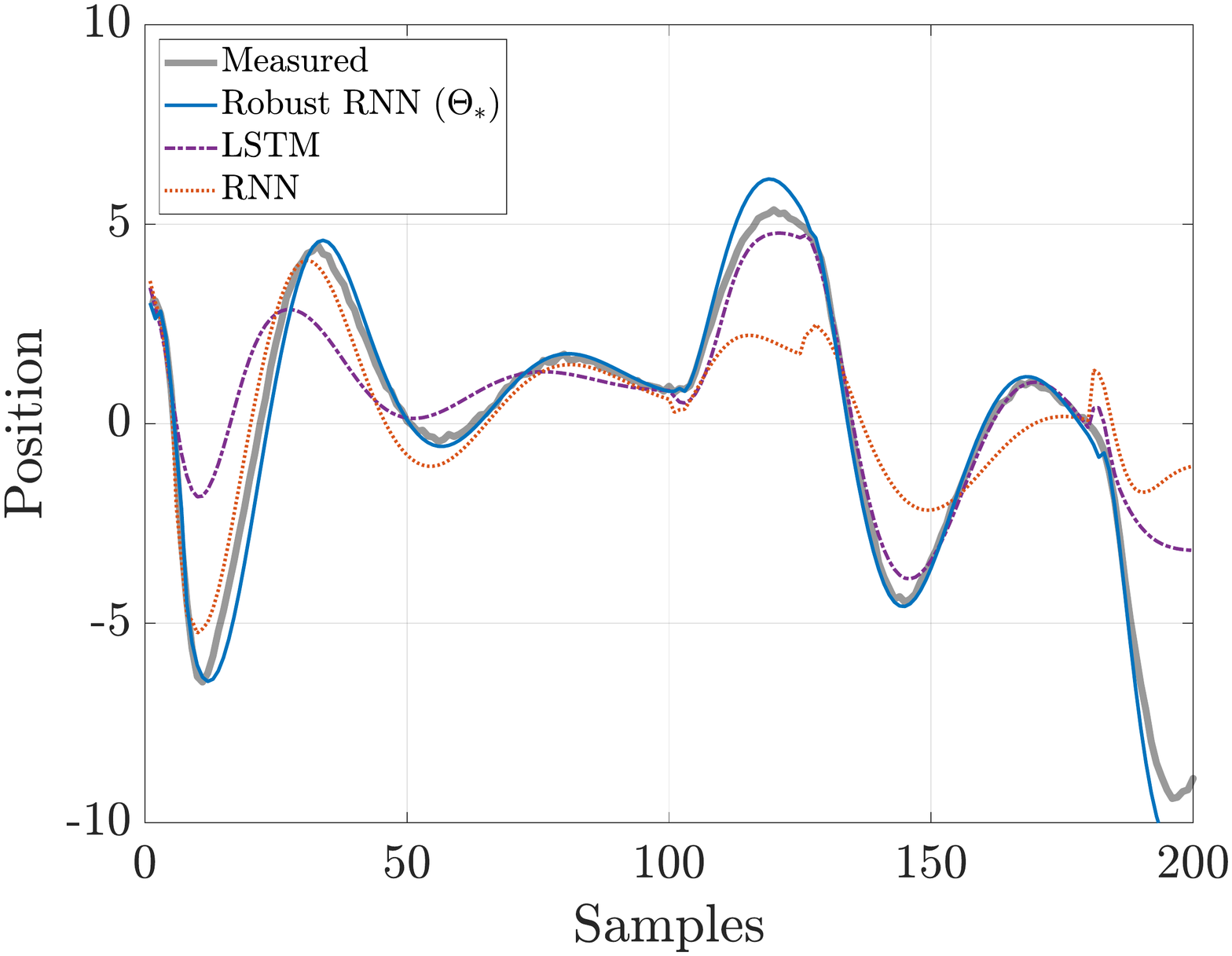}
\caption{\label{fig:trajectory_amp10} Simulation of models when $\sigma_u=10$, the robust RNN greatly outperforms the other models. }
%	\end{subfigure}
\end{figure}

%\begin{figure}
%	\centering
%	\includegraphics[width=0.95\linewidth, trim={1cm, 7cm, 2cm, 8cm}, clip]{median_comparison.pdf}
%	\caption{\label{fig:comparison}Model performance versus input distribution parameter. \tb{fuse with Figure 4}}
%\end{figure}